\documentclass{article} 
\usepackage[final]{colm2026_conference}

\DeclareFontFamily{OMS}{zi4}{}
\DeclareFontShape{OMS}{zi4}{m}{n}{<->ssub*cmsy/m/n}{}
\DeclareFontShape{OMS}{zi4}{b}{n}{<->ssub*cmsy/b/n}{}

\usepackage{microtype}
\usepackage{hyperref}
\usepackage{url}
\usepackage{booktabs}

\usepackage{amsmath,amssymb,amsthm}
\usepackage{algorithm,algpseudocode}
\newcommand{\ind}{\mathbf{1}}
\newtheorem{lemma}{Lemma}
\DeclareMathOperator*{\argmax}{arg\,max}
\usepackage{multirow}
\usepackage{graphicx} 
\usepackage{tcolorbox}
\usepackage{placeins}
\usepackage{lineno}

\definecolor{darkblue}{rgb}{0, 0, 0.5}
\hypersetup{colorlinks=true, citecolor=darkblue, linkcolor=darkblue, urlcolor=darkblue}

\title{Judge, Retrieve, or Abstain: Uncertainty-Guarded LLM Judging with Provable Risk Guarantees}

\author{%
\textbf{Sher Badshah}$^{1,2}$, \textbf{Ali Emami}$^{3}$, \textbf{Hassan Sajjad}$^{1,2}$ \\
$^{1}$Dalhousie University \quad $^{2}$New York University Abu Dhabi \quad $^{3}$Emory University \\
\texttt{\{sh545346,hsajjad\}@dal.ca} \\
\texttt{ali.emami@emory.edu}%
}

\begin{document}

\ifcolmsubmission
\linenumbers
\fi

\maketitle

\begin{abstract}
Using LLMs as judges has become standard practice for evaluating model outputs at scale. This is particularly common for subjective, open-ended tasks such as assessing helpfulness or alignment, where no single reference answer exists. However, objective tasks introduce a distinct reliability challenge for reference-free LLM judging. In the absence of a reference answer, the judge evaluates factual correctness either through its parametric knowledge or through tool augmentation. Although the former enables efficient evaluation, the judge may hallucinate or lack sufficient evidence for its verdict. Conversely, tool augmentation can provide additional evidence but introduces extra computational cost and requires an appropriate mechanism to determine when and how that evidence should be used reliably. More importantly, neither approach alone provides formal control over the risk of accepted verdicts or guarantees their reliability at a specified level. We propose a risk-controlled framework that calibrates uncertainty thresholds on a held-out set so that the false discovery rate among accepted verdicts remains below a user-specified level~$\alpha$ with high probability, using finite-sample Clopper--Pearson intervals. When the parametric mode is not sufficiently confident, the instance is routed to a retrieval-augmented mode, where the judge gathers web evidence and re-evaluates the instance under a second calibrated threshold. The finite-sample guarantee carries over to this two-threshold routing without additional assumptions. Across open-domain QA benchmarks and judges of varying scales, the framework maintains the target error rate while achieving substantially higher coverage than single-mode baselines.
\end{abstract}

\section{Introduction}
\label{sec:intro}
LLM-as-a-judge has emerged as a scalable alternative to human evaluation for assessing language model outputs~\citep{NEURIPS2023_91f18a12}. It enables practitioners to evaluate thousands of model outputs at a fraction of the cost and time required for human annotation~\citep{gu2024survey, tan2025judgebenchbench}. However, most existing work employs LLM judges for subjective evaluation~\citep{badshah-clev}, such as scoring outputs for helpfulness, harmlessness, or style in reward modeling, where no single correct answer exists and approximate agreement suffices~\citep{yuan-etal-2024-r, liu-etal-2023-g}. Extending LLM-as-a-judge to objective, factual evaluation remains an open challenge. 

In factual and open-domain QA, answers are either correct or incorrect, and the judge must make this determination reliably. However, obtaining ground-truth annotations for this purpose is expensive at scale~\citep{chianglee2023large, manas2024improving}, requires domain expertise for specialized questions, and is entirely infeasible when LLMs operate as autonomous agents answering open-ended queries in real-world environments where no pre-collected labels exist. LLM-as-a-judge offers a reference-free alternative that eliminates this dependency. Nevertheless, utilizing LLM judges to assess correctness without explicit references introduces two key reliability challenges~\citep{schroeder2025trustllm, badshah2025ref}.

First, LLM judges are bounded by their training data and may lack knowledge of recent events, rare entities, or specialized domains~\citep{10.1162/tacl_a_00754}. Such limitations can prevent reliable distinction between correct answers and plausible fabrications. Second, even if the judge has the relevant knowledge, hallucinations may still cause it to produce confident but incorrect verdicts~\citep{10.1145/3703155}. In both cases, the failures are silent as the judge provides no indication that its verdict is unreliable. Without a mechanism to quantify and act on this uncertainty, incorrect evaluations can propagate undetected into downstream decisions.

Selective evaluation provides a natural mitigation strategy by allowing the judge to abstain when uncertain and produce a verdict only when it is likely to be correct. Uncertainty quantification methods can identify unreliable verdicts~\citep{10.1145/3744238}. However, applying a heuristic threshold to these scores provides no formal guarantee on the error rate among accepted evaluations~\citep{scope}. Recent work on risk-controlled selective prediction addresses this gap by calibrating statistically valid thresholds that bound the False Discovery Rate (FDR) among selected outputs with high probability~\citep{safer, wang2025coin}. However, existing methods are designed to filter examinee answers in Question-Answering (QA) tasks rather than control errors in a meta-evaluation setting. Moreover, earlier methods treat abstention as the only recourse for uncertain instances. This discards cases where the judge's uncertainty stems from a knowledge gap or hallucination that could be resolved with additional evidence~\citep{tale_sb}.

In this paper, we propose a risk-controlled selective evaluation framework for LLM-as-a-judge that addresses the above limitations. Inspired by LLM-based agents that leverage external tools to overcome their parametric limitations, we equip the judge with web search as a tool to bridge its knowledge gaps. Our framework operates in two modes. In Mode~1, the judge model evaluates using only its parametric knowledge without any tool-augmentation. If the judge is confident, the verdict is accepted. If uncertain, the instance is routed to Mode~2 where the judge retrieves relevant evidence from the web and re-evaluates with this augmented context. If the judge remains uncertain after retrieval, it abstains and flags the instance for human review. Both thresholds are jointly calibrated on a held-out set with known ground-truth labels via Clopper--Pearson upper confidence bounds~\citep{be7c0fd0_cp}. We prove that the finite-sample FDR guarantee for a fixed threshold pair extends to the two-threshold routing setting without additional distributional assumptions. Our contributions are:
\begin{itemize}
\item We formulate risk-controlled selective evaluation for LLM-as-a-judge on factual QA.
\item We introduce a two-mode routing mechanism that equips the judge with adaptive web retrieval to resolve uncertainty before resorting to abstention.
\item We show that the Bernoulli structure required for the Clopper--Pearson bound extends to the two-threshold routing setting.
\item Experiments across open-domain QA benchmarks and judges of varying scales demonstrate that adaptive retrieval recovers substantial coverage over single-mode baselines while controlling the judge's error rate.
\end{itemize}
\section{Related Work}
\label{sec:related}

\paragraph{LLM-as-a-Judge.}
Using LLMs to evaluate the outputs of other models has gained widespread adoption as a cost-effective alternative to human annotation. Strong LLMs can approximate human preferences on subjective dimensions~\citep{liu-etal-2023-g}. As a result, they have been adopted for reward modeling and reinforcement learning~\citep{son2024llmasajudgerewardmodel}. More recently, LLM judges have been extended to objective tasks such as factual evaluation. SAFE~\citep{NEURIPS2024_937ae0e8} uses an LLM agent with web search to verify individual facts in long-form responses. Similarly, SAGE~\citep{tale_sb} employs a tool-augmented agent that iteratively retrieves and synthesizes external evidence for reference-free QA evaluation. These works improve judge accuracy but provide no formal guarantee on the reliability of verdicts.

\paragraph{Uncertainty quantification for LLM evaluation.}
Existing methods generally fall into two categories. White-box methods estimate uncertainty from the judge's token probabilities. The most common choice is the predictive entropy of the output distribution~\citep{malinin2021uncertainty}. These probabilities are often well calibrated with respect to correctness~\citep{kadavath2022languagemodelsmostlyknow}. Black-box methods, in contrast, estimate uncertainty by querying the judge repeatedly and measuring how consistent its verdicts are~\citep{002WSLCNCZ23, wang-etal-2024-conu}. Both signals correlate with judge accuracy and can flag low-confidence verdicts~\citep{scope}. However, they remain heuristic. For instance, a threshold chosen by hand or tuned on a validation set gives no formal guarantee on the error rate among accepted verdicts.

\paragraph{Uncertainty-based routing and tool use.}
A parallel line of work uses uncertainty to determine when additional resources are needed. FrugalGPT~\citep{chen2023frugalgptuselargelanguage} and AutoMix~\citep{NEURIPS2024_ecda225c} use confidence to route queries across models with different capabilities, trading off cost against quality. In contrast, FLARE~\citep{jiang-etal-2023-active} uses token-level uncertainty to decide when to retrieve external information during generation. More recently, SAGE~\citep{tale_sb} introduced search-augmented retrieval for LLM judges, where the judge acts as an agent that gathers targeted evidence when evaluating free-form answers. However, these approaches often rely on unconditional retrieval, which introduces unnecessary computation. Furthermore, they do not provide formal guarantees on the error rate of the resulting evaluations.

\paragraph{Risk-controlled selective evaluation.}
Conformal prediction provides distribution-free, finite-sample coverage guarantees by constructing prediction sets calibrated on held-out data~\citep{angelopoulos2023gentle}. Conformal risk control generalizes this beyond coverage to bound arbitrary loss functions, including the FDR, at user-specified levels~\citep{angelopoulos2024conformal}. Recent work has applied these guarantees to question answering. SConU builds prediction sets that cover admissible answers with high probability~\citep{wang-etal-2025-sconu}. Similarly, COIN~\citep{wang2025coin} calibrates an uncertainty threshold that bounds the FDR among accepted answers using Clopper--Pearson upper confidence bounds~\citep{be7c0fd0_cp}. Closer to our setting, Trust or Escalate~\citep{jung2025trust} and SCOPE~\citep{scope} adapt risk-controlled selective prediction to pairwise evaluation and provide guarantee on human agreement among accepted judgments. However, both target subjective preference tasks where the judge compares two responses rather than evaluating factual correctness (i.e., objective evaluation).

Our framework introduces risk-controlled retrieval-augmented evaluation, where retrieval is selectively activated based on calibrated uncertainty thresholds. Unlike unconditional retrieval approaches, it provides formal FDR guarantees for the accepted verdicts.
\section{Methodology}
\label{sec:method}

We propose a risk-controlled framework for LLM-as-a-judge evaluation. The framework calibrates uncertainty thresholds on a held-out set with human labels so that, among all instances the judge evaluates, the probability of producing an incorrect verdict is bounded by a user-specified risk level $\alpha$~\citep{angelopoulos2023gentle, angelopoulos2024conformal}. When the judge is uncertain, the framework first attempts to improve the evaluation through web retrieval; if the judge remains uncertain after retrieval, it abstains and flags the instance for human review~\footnote{We do not implement human review; the term indicates that these instances are excluded from automated evaluation.}.

\subsection{Problem Formulation}
\label{sec:formulation}

\paragraph{Task.}
Let $F\!:\!\mathcal{X}\!\to\!\mathcal{Y}$ be an examinee LLM that, given a question $\mathbf{x} \in \mathcal{X}$, generates a candidate answer $\hat{y} \in \mathcal{Y}$, and let $y^* \in \mathcal{Y}$ be the corresponding ground-truth answer. A judge LLM $\mathcal{J}\!:\!\mathcal{X}\!\times\!\mathcal{Y}\!\to\!\{0,1\}$ evaluates whether $\hat{y}$ correctly answers~$\mathbf{x}$, producing a binary verdict $v \in \{0,1\}$, where $v\!=\!1$ means the judge deems $\hat{y}$ correct.

\paragraph{Admission function.}
The ground-truth evaluation label $e^* \in \{0,1\}$ indicates whether $\hat{y}$ is truly correct. It can be obtained from human annotation, or by applying a task-specific relevance function $\mathcal{A}_{\mathrm{ref}}\!:\!\mathcal{Y}\!\times\!\mathcal{Y}\!\to\![0,1]$ with a threshold~$\lambda_A$:
\begin{equation}
\label{eq:gt_label}
  e^* = \ind\!\left[\mathcal{A}_{\mathrm{ref}}(\hat{y},\, y^*) \geq \lambda_A\right].
\end{equation}
A verdict~$v$ is admissible if it matches the ground truth:
\begin{equation}
\label{eq:admission}
  \mathcal{A}(v,\, e^*) = \ind\!\left[v = e^*\right].
\end{equation}

\paragraph{Selective evaluation.}
Rather than always outputting a verdict, the judge is allowed to abstain when it is uncertain. We define a failure indicator $W = \ind[v \neq e^*]$. Given an uncertainty score $U$ (defined in \S\ref{sec:uncertainty}), the judge outputs a verdict only when $U \leq t$ for some threshold~$t$. The false discovery rate (FDR) is:
\begin{equation}
\label{eq:fdr}
  R(t) = \mathbb{E}\!\left[W \mid U \leq t\right],
\end{equation}
i.e., the error rate among instances the judge chooses to evaluate.

\paragraph{Goal.}
For a user-specified risk level $\alpha \in (0,1)$ and significance level $\delta$ (e.g., $0.05$), we seek a threshold~$t$ such that:
\begin{equation}
\label{eq:goal}
  \Pr\!\left(R(t) \leq \alpha\right) \geq 1 - \delta,
\end{equation}
where the probability is over the randomness of the calibration data. With high confidence ($1\!-\!\delta$), the error rate among the judge's non-abstained outputs is at most~$\alpha$.

\paragraph{Calibration data.}
We hold out $N$ calibration samples $\mathcal{D}_{\mathrm{cal}} = \{(\mathbf{x}_i, \hat{y}_i, y^*_i)\}_{i=1}^{N}$ with known ground-truth labels $e_i^*$, drawn i.i.d.\ from the data-generating distribution~$\mathcal{D}$.

\subsection{Uncertainty Quantification (UQ)}
\label{sec:uncertainty}

Accurate uncertainty estimation is critical for selective evaluation: the better the uncertainty score separates correct from incorrect verdicts, the higher the coverage (fewer abstentions) at a given risk level. 

\subsection{Risk-Bounded Selective Evaluation}
\label{sec:single}

We now describe how to calibrate the uncertainty threshold to satisfy the guarantee in Eq.~\eqref{eq:goal}. This section presents the base case where the judge uses only its parametric knowledge (no retrieval); we extend to retrieval-augmented judging in \S\ref{sec:routing}.

\paragraph{Empirical failure analysis.}
For each calibration instance, we obtain the judge's verdict $v_i$ and uncertainty score $U_i$. Given a candidate threshold~$t$, we compute the number of instances the judge would evaluate (selection size):
\begin{equation}
\label{eq:mcal}
  \hat{m}_{\mathrm{cal}}(t) = \sum_{i=1}^{N} \ind\!\left\{U_i \leq t\right\},
\end{equation}
and the number of errors among them (error count):
\begin{equation}
\label{eq:wcal}
  \hat{w}_{\mathrm{cal}}(t) = \sum_{i=1}^{N} \ind\!\left\{U_i \leq t\right\}\!\cdot\!(1 - c_i),
\end{equation}
where $c_i = \mathcal{A}(v_i, e^*_i)$ is the correctness indicator. The empirical FDR is $\hat{r}_{\mathrm{cal}}(t) = \hat{w}_{\mathrm{cal}}(t) / \hat{m}_{\mathrm{cal}}(t)$.

\paragraph{Upper confidence bound (UCB).}
The empirical FDR is an unbiased point estimate of the true failure rate, but may underestimate it on any given calibration set due to sampling variability. To obtain a rigorous bound, we leverage the Clopper--Pearson exact method. Since the calibration instances are i.i.d.\ and the selection event $\{U_i \leq t\}$ is a deterministic function of $(\mathbf{x}_i, \hat{y}_i)$, the failure indicators of selected instances are i.i.d.\ Bernoulli with parameter $R(t)$. We can therefore apply the one-sided Clopper--Pearson bound:
\begin{equation}
\label{eq:ucb}
  \hat{R}^{\mathrm{upper}}_{1-\delta}(t) =
  \mathrm{BetaInv}\!\big(
  1 - \delta;\;
  \hat{w}_{\mathrm{cal}}(t) + 1,\;
  \hat{m}_{\mathrm{cal}}(t) - \hat{w}_{\mathrm{cal}}(t)
  \big),
\end{equation}
which satisfies $\Pr(R(t) \leq \hat{R}^{\mathrm{upper}}_{1-\delta}(t)) \geq 1 - \delta$ (Appendix~\ref{app:proof_ucb}).

\paragraph{Threshold selection.}
We select the largest threshold whose UCB stays below the risk level $\alpha$:
\begin{equation}
\label{eq:that_single}
  \hat{t} = \sup\big\{t : \hat{R}^{\mathrm{upper}}_{1-\delta}(t') \leq \alpha \;\text{for all}\; t' \leq t\big\}.
\end{equation}
This maximizes the number of instances the judge evaluates while targeting Eq.~\eqref{eq:goal}.

\subsection{Adaptive Retrieval via Two-Mode Routing}
\label{sec:routing}
Inspired by LLM-based agents that leverage external tools to overcome the base model limitations, we equip the judge with web search as a tool to bridge its knowledge gaps. When the judge is not confident based on its parametric knowledge alone, it retrieves relevant evidence from the web and re-evaluates with this augmented context. This recovers coverage on exactly the instances that the single-mode framework would abstain on, turning abstention into an opportunity for informed evaluation.

\paragraph{Two evaluation modes.}
For each instance $(\mathbf{x}_i, \hat{y}_i)$, we define:
\begin{itemize}
  \item \textbf{Mode~1} (direct): $\mathcal{J}$ evaluates using only its internal knowledge, producing verdict $v_1^{(i)}$ with uncertainty $U_1^{(i)}$.
  \item \textbf{Mode~2} (retrieval-augmented): The system queries a web search engine using the question $\mathbf{x}_i$, retrieving the top-$k$ results, each containing a title, text snippet, and source URL. These are concatenated into an evidence passage $\mathcal{E}_i$ appended to the judge's prompt. The judge $\mathcal{J}$ re-evaluates with this context, producing $v_2^{(i)}$ with uncertainty $U_2^{(i)}$.
\end{itemize}
We denote the correctness of each mode as $c_k^{(i)} = \mathcal{A}(v_k^{(i)}, e^*_i)$ for $k \in \{1, 2\}$.

\paragraph{Routing logic.}
At test time, given calibrated thresholds $(\hat{t}_1, \hat{t}_2)$:
\begin{equation}
\label{eq:routing}
  \text{output} =
  \begin{cases}
    v_1, & \text{if } U_1 \leq \hat{t}_1, \\
    v_2, & \text{if } U_1 > \hat{t}_1 \;\text{and}\; U_2 \leq \hat{t}_2, \\
    \textsc{abstain}, & \text{otherwise.}
  \end{cases}
\end{equation}
Retrieval is invoked only when Mode~1 is uncertain, so confident instances incur zero retrieval cost.

\paragraph{Joint threshold calibration.}
During calibration, we pre-compute both modes for all $N$ instances in $\mathcal{D}_{\mathrm{cal}}$. For a candidate pair $(t_1, t_2)$, the selection size counts instances accepted via either mode:
\begin{equation}
\label{eq:mcal_route}
  \hat{m}_{\mathrm{cal}}(t_1, t_2) = \sum_{i=1}^{N} \Big[\ind\!\left\{U_1^{(i)} \leq t_1\right\} + \ind\!\big\{U_1^{(i)} > t_1 \;\text{and}\; U_2^{(i)} \leq t_2\big\}\Big].
\end{equation}
The error count among selected instances is:
\begin{equation}
\label{eq:wcal_route}
  \hat{w}_{\mathrm{cal}}(t_1, t_2) = \sum_{i=1}^{N} \Big[\ind\!\{U_1^{(i)} \leq t_1\}(1 - c_1^{(i)}) + \ind\!\{U_1^{(i)} > t_1 \wedge U_2^{(i)} \leq t_2\}(1 - c_2^{(i)})\Big].
\end{equation}
The empirical FDR and UCB follow as before:
\begin{align}
  \hat{r}_{\mathrm{cal}}(t_1, t_2) &= \frac{\hat{w}_{\mathrm{cal}}(t_1, t_2)}{\hat{m}_{\mathrm{cal}}(t_1, t_2)}, \label{eq:fdr_route} \\[4pt]
  \hat{R}^{\mathrm{upper}}_{1-\delta}(t_1, t_2) &= \mathrm{BetaInv}\!\big(1 - \delta;\; \hat{w}_{\mathrm{cal}} + 1,\; \hat{m}_{\mathrm{cal}} - \hat{w}_{\mathrm{cal}}\big). \label{eq:ucb_route}
\end{align}
We select the threshold pair that maximizes coverage subject to the risk constraint:
\begin{equation}
\label{eq:that_route}
  (\hat{t}_1, \hat{t}_2) = \argmax_{t_1,\, t_2} \hat{m}_{\mathrm{cal}}(t_1, t_2) \quad \text{s.t.} \quad \hat{R}^{\mathrm{upper}}_{1-\delta}(t_1, t_2) \leq \alpha.
\end{equation}
In practice, this is solved by search over the sorted unique uncertainty values on the calibration set. For each candidate $t_1$, cumulative sums enable evaluating all $t_2$ candidates in $O(|\mathcal{T}_2|)$ time, yielding $O(|\mathcal{T}_1| \cdot |\mathcal{T}_2|)$ total complexity (Algorithm~\ref{alg:calibration}). Because this search evaluates many threshold pairs, we additionally report a conservative variant that corrects for it via a Bonferroni adjustment ($\delta' = \delta/(|\mathcal{T}_1|\,|\mathcal{T}_2|)$) (see Appendix~\ref{app:bonferroni}).

\subsection{Theoretical Guarantee}
\label{sec:guarantee}

The central theoretical question is whether the single-mode Clopper–Pearson guarantee (Eq.~\ref{eq:goal}) extends to the two-mode routing setting. It does, because the routing decision for each instance is entirely determined by its uncertainty scores and the fixed thresholds, introducing no additional randomness.

\begin{lemma}[Bernoulli structure under routing]
\label{lem:routing}
Let $(\mathbf{x}_i, \hat{y}_i, y^*_i)$ be i.i.d.\ samples from $\mathcal{D}$, and let $(t_1, t_2)$ be fixed thresholds. Under the routing policy (Eq.~\ref{eq:routing}), the failure indicators $\{W_i\}$ over the selected subset are i.i.d.\ $\mathrm{Bernoulli}(R(t_1, t_2))$.
\end{lemma}

\noindent The proof (Appendix~\ref{app:proof_routing}) shows that the selection indicator and the correctness of selected instances are both deterministic functions of $(\mathbf{x}_i, \hat{y}_i, y^*_i)$. Since the original data is i.i.d., conditioning on selection preserves independence and identical distribution. Combined with the Clopper--Pearson bound, this yields a guarantee for any fixed threshold pair:
\begin{equation}
\label{eq:guarantee_route}
  \boxed{\;\Pr\!\left(R(t_1, t_2) \leq \hat{R}^{\mathrm{upper}}_{1-\delta}(t_1, t_2)\right) \geq 1 - \delta.\;}
\end{equation}
In particular, any fixed pair whose bound satisfies $\hat{R}^{\mathrm{upper}}_{1-\delta}(t_1,t_2) \leq \alpha$ controls the error rate at $\alpha$ with confidence $1-\delta$. 

\begin{figure}[t]
\begin{minipage}[t]{0.52\textwidth}
\begin{algorithm}[H]
\small
\caption{Calibration (offline)}
\label{alg:calibration}
\begin{algorithmic}[1]
\Require Calibration set $\mathcal{D}_{\mathrm{cal}}$, judge $\mathcal{J}$
\Require Risk level $\alpha$, significance level $\delta$
\Ensure Calibrated thresholds $\hat{t}_1$, $\hat{t}_2$
\Statex
\For{each $(\mathbf{x}_i, \hat{y}_i, y_i^*) \in \mathcal{D}_{\mathrm{cal}}$}
    \State $v_1^{(i)}, U_1^{(i)} \gets \mathcal{J}(\mathbf{x}_i, \hat{y}_i)$
    \State $\mathcal{E}_i \gets \textsc{Retrieve}(\mathbf{x}_i)$
    \State $v_2^{(i)}, U_2^{(i)} \gets \mathcal{J}(\mathbf{x}_i, \hat{y}_i, \mathcal{E}_i)$
    \State $c_1^{(i)} \!\gets\! \mathcal{A}(v_1^{(i)}, e_i^*)$;\; $c_2^{(i)} \!\gets\! \mathcal{A}(v_2^{(i)}, e_i^*)$
\EndFor
\Statex
\State $\mathcal{T}_1 \gets$ sorted unique $\{U_1^{(i)}\}_{i=1}^{N}$
\State $\mathcal{T}_2 \gets$ sorted unique $\{U_2^{(i)}\}_{i=1}^{N}$
\State $\hat{t}_1, \hat{t}_2 \gets \textsc{null}$;\; $m^* \gets 0$
\For{$t_1$ in $\mathcal{T}_1$}
    \For{$t_2$ in $\mathcal{T}_2$}
        \State Compute $\hat{m}_{\mathrm{cal}},\, \hat{w}_{\mathrm{cal}}$
        \State $R^+ \!\gets\! \mathrm{BetaInv}(1\!-\!\delta;\, \hat{w}_{\mathrm{cal}}\!+\!1,\, \hat{m}_{\mathrm{cal}}\!-\!\hat{w}_{\mathrm{cal}})$
        \If{$R^+ \leq \alpha$ \textbf{and} $\hat{m}_{\mathrm{cal}} > m^*$}
            \State $\hat{t}_1 \!\gets\! t_1$; $\hat{t}_2 \!\gets\! t_2$; $m^* \!\gets\! \hat{m}_{\mathrm{cal}}$
        \EndIf
    \EndFor
\EndFor
\State \Return $\hat{t}_1$, $\hat{t}_2$
\end{algorithmic}
\end{algorithm}
\end{minipage}%
\hfill
\begin{minipage}[t]{0.45\textwidth}
\begin{algorithm}[H]
\small
\caption{Test-time evaluation (online)}
\label{alg:test}
\begin{algorithmic}[1]
\Require Instance $(\mathbf{x}, \hat{y})$, judge $\mathcal{J}$
\Require Thresholds $\hat{t}_1$, $\hat{t}_2$
\Ensure Verdict $v \in \{0, 1\}$ or \textsc{abstain}
\Statex
\State $v_1, U_1 \gets \mathcal{J}(\mathbf{x}, \hat{y})$
\If{$U_1 \leq \hat{t}_1$}
    \State \Return $v_1$ \Comment{Confident}
\EndIf
\Statex
\State $\mathcal{E} \gets \textsc{Retrieve}(\mathbf{x})$
\State $v_2, U_2 \gets \mathcal{J}(\mathbf{x}, \hat{y}, \mathcal{E})$
\If{$U_2 \leq \hat{t}_2$}
    \State \Return $v_2$ \Comment{Retrieval helped}
\EndIf
\Statex
\State \Return \textsc{abstain} \Comment{Human review}
\end{algorithmic}
\end{algorithm}
\vspace{0.5em}
\small
\textbf{Cost analysis.} Let $p$ be the fraction of instances accepted by Mode~1. Test-time retrieval is invoked for only $(1\!-\!p)$ of instances. At $\alpha\!=\!0.20$ with Qwen3-14B, $p\!\approx\!0.45$ on TriviaQA, so nearly half of evaluations incur zero retrieval cost. On harder benchmarks $p$ decreases, but the framework still avoids retrieval for the most confident instances.
\end{minipage}
\end{figure}

\section{Experiments}
\label{sec:experiments}

\paragraph{Datasets and models.}
We evaluate on four open-domain QA benchmarks: TriviaQA~\citep{joshi2017triviaqa}, Natural Questions (NQ)~\citep{kwiatkowskietal2019}, HotpotQA~\citep{yang2018hotpotqa}, and PopQA~\citep{mallen2023llm_memorization}. From each benchmark's standard held-out split, we sample 2{,}000 instances using seed~42.

We use two candidate models: Qwen3-8B~\citep{yang2025qwen3technicalreport} and LLaMA-3.1-70B~\citep{grattafiori2024llama3herdmodels}. For judges, we employ four models: Qwen3-4B, Qwen3-8B, and Qwen3-14B~\citep{yang2025qwen3technicalreport} and LLaMA-3.1-8B-Instruct~\citep{grattafiori2024llama3herdmodels}. Each judge outputs a binary verdict (\texttt{True}/\texttt{False}) and a brief explanation.

\paragraph{Admission functions.}
To evaluate whether the judge's verdict is correct, we first determine whether the candidate's answer is actually right by comparing it against the reference answer. This ground-truth label~$e^*$ is computed using three admission functions: (1)~Exact Match~(EM), (2)~Token~F1 (threshold~$\geq 0.5$), and (3)~LLM Evaluator~\citep{badshah2025ref, badshah-clev}, which uses Qwen2.5-7B-Instruct~\citep{qwen2025qwen25technicalreport} as a reference-based evaluator to judge semantic equivalence.

\paragraph{Uncertainty quantification.}
We use predictive entropy~(PE)~\citep{malinin2021uncertainty} as the uncertainty measure. PE is computed from the judge's token-level log-probabilities at the verdict position via a single forward pass~\citep{kadavath2022languagemodelsmostlyknow}. When the judge's hidden states are available, PE can be replaced by supervised scores without changing the guarantee. We report a probe and a LoRA+Prompt variant in Appendix~\ref{app:uqmenu}.

\paragraph{Retrieval.}
For Mode~2, we query a web search engine using the original question (i.e., the question provided to the candidate) and retrieve the top-$k$ results, with $k=3$ as the default setting. The retrieved results are concatenated into a single passage and appended to the judge prompt as additional context. We further study the effect of retrieval depth by varying $k \in \{3,6,9\}$ on one candidate--judge pair (LLaMA-3.1-70B evaluated by Qwen3-8B).

\paragraph{Calibration protocol.}
We run 100 random 50/50 calibration/test splits using the Clopper--Pearson exact binomial bound. We set $\delta\!=\!0.05$ and evaluate $\alpha \in \{0.05, 0.10, 0.15, 0.20, 0.25\}$. Appendix~\ref{app:calibsens} reports sensitivity to both $\delta$ and the calibration-set size.

\paragraph{Evaluation metrics.}
We report: (1)~FDR, the fraction of accepted instances with incorrect verdicts ($w_\text{test}/m_\text{test}$); (2)~Coverage, the fraction of instances accepted rather than abstained on ($m_\text{test}/N_\text{test}$); and (3)~Routing distribution which is the split across Mode~1, Mode~2, and abstention.

\section{Results}
\label{sec:results}

\begin{figure*}[t]
\centering
\includegraphics[width=\textwidth]{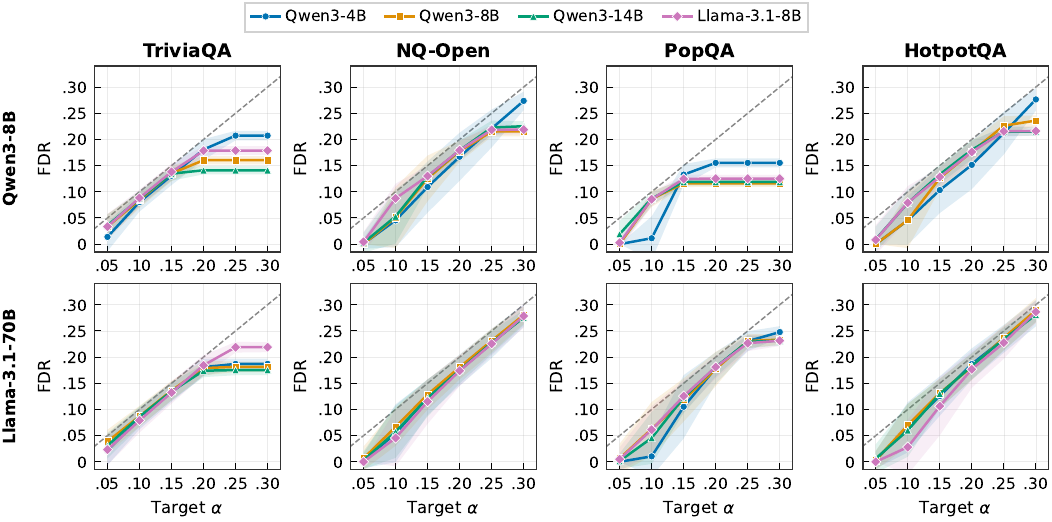}
\caption{Test-time FDR (mean$\pm$std over 100 splits) across all 32 configurations. Rows correspond to the two candidate models; columns to the four datasets. Each line represents a different judge (see legend). The dashed diagonal marks FDR~$= \alpha$.}
\label{fig:fdr_all}
\end{figure*}

We evaluate whether the framework delivers on its two central promises: (1)~the FDR guarantee holds empirically across judges, candidates, and datasets, and (2)~adaptive retrieval recovers coverage that would otherwise be lost to abstention. We use predictive entropy and LLM evaluator for the main results, and additionally analyze the sensitivity of our framework to the choice of admission function and retrieval depth (\S\ref{sec:sensitivity}).

\subsection{FDR control holds across all configurations}
\label{sec:fdr_results}

The most basic requirement of our framework is that the empirical FDR remains at or below the target~$\alpha$. Figure~\ref{fig:fdr_all} tests this across all candidate--judge--dataset configurations. In every case, the FDR stays at or below~$\alpha$, satisfying the finite-sample guarantee (Eq.~\ref{eq:guarantee_route}).

Three observations are noteworthy. First, all four judges provide valid FDR control regardless of their parameter count, showing that the guarantee does not require a strong judge. Second, FDR tracks close to~$\alpha$ with narrow confidence bands ($\pm$0.01--0.02), which indicates stable calibration across 100 random splits. Third, the gap between FDR and~$\alpha$ is smallest on TriviaQA and PopQA, which are relatively easy tasks where judges are most accurate. This demonstrates that the framework is tightest when the judge has sufficient knowledge to evaluate.

\subsection{Coverage scales with risk tolerance and judge capability}
\label{sec:coverage_results}

Table~\ref{tab:coverage} reports coverage across five risk levels. Coverage scales smoothly with~$\alpha$ and is governed by two factors: dataset difficulty and judge capability. On easy benchmarks (TriviaQA, PopQA), even the smallest judge reaches near-full coverage by $\alpha\!=\!0.25$, while on knowledge-intensive tasks (NQ-Open, HotpotQA) only the larger judges (Qwen-14B, Llama-8B) surpass 80\% at $\alpha\!=\!0.20$ for the Qwen-8B candidate --- Qwen-4B remains below 15\% on these same benchmarks. Evaluating the stronger Llama-70B candidate is consistently harder: coverage drops across all judges compared to Qwen-8B, with the best judge reaching only 53\% on HotpotQA at $\alpha\!=\!0.20$. This gap reflects both the answer diversity of the 70B model and the increased knowledge demands it places on the judge.

\begin{table*}[h]
\centering
\caption{Coverage of the joint two-mode framework (two-threshold routing with retrieval) at $\alpha \in \{.05, .10, .15, .20, .25\}$ (mean over 100 splits) across all datasets.}
\label{tab:coverage}
\resizebox{\textwidth}{!}{%
\begin{tabular}{ll|ccccc|ccccc|ccccc|ccccc}
\toprule
 & & \multicolumn{5}{c|}{TriviaQA} & \multicolumn{5}{c|}{NQ-Open} & \multicolumn{5}{c|}{PopQA} & \multicolumn{5}{c}{HotpotQA} \\
Cand. & Judge & .05 & .10 & .15 & .20 & .25 & .05 & .10 & .15 & .20 & .25 & .05 & .10 & .15 & .20 & .25 & .05 & .10 & .15 & .20 & .25 \\
\midrule
\multirow{4}{*}{\rotatebox{90}{\scriptsize Qwen-8B}}
 & Qwen-4B   & .03 & .39 & .66 & .91 & 1.0 & .00 & .03 & .08 & .14 & .28 & .00 & .01 & .91 & 1.0 & 1.0 & .00 & .04 & .09 & .14 & .30 \\
 & Qwen-8B   & .12 & .63 & .91 & 1.0 & 1.0 & .00 & .02 & .24 & .82 & 1.0 & .00 & .87 & 1.0 & 1.0 & 1.0 & .00 & .04 & .49 & .74 & .96 \\
 & Qwen-14B  & .21 & .78 & .98 & 1.0 & 1.0 & .00 & .05 & .56 & .81 & .99 & .19 & .89 & 1.0 & 1.0 & 1.0 & .02 & .26 & .64 & .86 & 1.0 \\
 & Llama-8B  & .13 & .61 & .86 & 1.0 & 1.0 & .00 & .20 & .61 & .83 & 1.0 & .02 & .88 & 1.0 & 1.0 & 1.0 & .01 & .30 & .65 & .85 & 1.0 \\
\midrule
\multirow{4}{*}{\rotatebox{90}{\scriptsize Llama-70B}}
 & Qwen-4B   & .05 & .37 & .80 & .98 & 1.0 & .00 & .05 & .18 & .32 & .48 & .00 & .00 & .11 & .35 & .92 & .00 & .09 & .20 & .34 & .55 \\
 & Qwen-8B   & .17 & .44 & .83 & .99 & 1.0 & .01 & .10 & .25 & .47 & .73 & .00 & .06 & .11 & .48 & .98 & .00 & .11 & .31 & .47 & .65 \\
 & Qwen-14B  & .12 & .46 & .85 & 1.0 & 1.0 & .00 & .05 & .23 & .51 & .75 & .00 & .04 & .31 & .77 & .99 & .00 & .07 & .31 & .53 & .76 \\
 & Llama-8B  & .04 & .32 & .67 & .86 & 1.0 & .00 & .02 & .13 & .32 & .59 & .00 & .07 & .23 & .77 & .98 & .00 & .02 & .11 & .27 & .51 \\
\bottomrule
\end{tabular}%
}
\end{table*}

\subsection{Adaptive retrieval recovers substantial coverage}
\label{sec:routing_results}
Figure~\ref{fig:routing} breaks down accepted instances by routing mode. It is depicted that retrieval is the dominant acceptance mechanism. Across various configurations, Mode~2 accounts for the majority of non-abstained verdicts. This does not necessarily mean the judge lacks the knowledge to evaluate correctly, but rather that its uncertainty signal in Mode~1 is not confident enough to satisfy the calibrated threshold. Crucially, these are instances where a single-mode framework would abstain entirely. By grounding the judge in retrieved evidence, Mode~2 reduces uncertainty below the second threshold and converts these abstentions into accepted verdicts. This trade-off also depends on the underlying judge model. Models with better-calibrated uncertainty estimates allow a larger fraction of instances to be resolved by Mode~1 directly. For instance, at $\alpha\!=\!0.20$ on TriviaQA, Qwen3-14B resolves 45\% of evaluations parametrically versus 9\% for Qwen3-4B. On harder tasks like HotpotQA and NQ-Open, all judges rely predominantly on retrieved evidence. 

Although retrieval substantially expands coverage, it does not always improve the judge's confidence. For 14--29\% of instances, retrieval increases rather than decreases uncertainty (Appendix~\ref{app:deltau}). The second calibrated threshold identifies these cases and routes them to abstention. Consequently, only retrieval-based verdicts that satisfy the desired risk level are accepted.

\begin{figure*}[htbp]
\centering
\includegraphics[width=\textwidth]{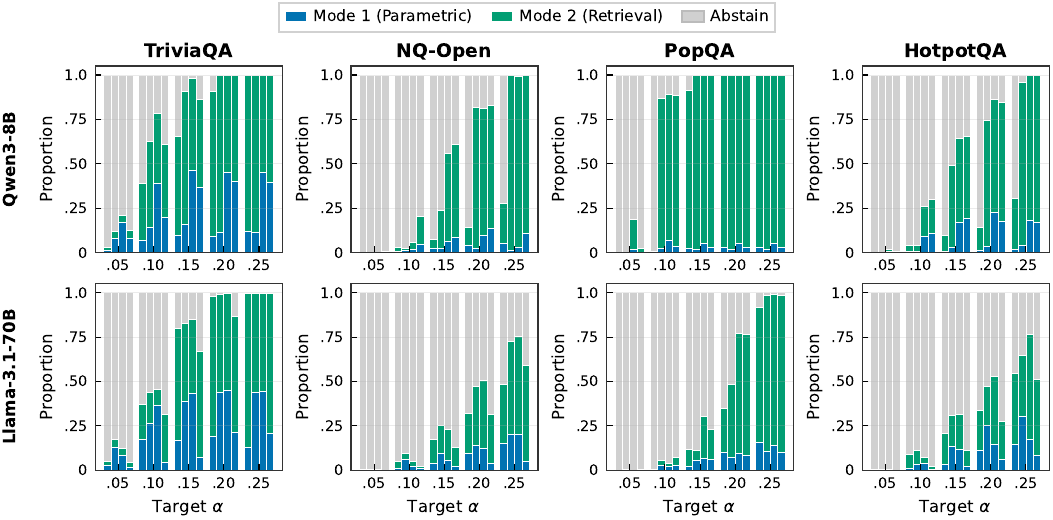}
\caption{Routing distribution across configurations. Within each group of four bars, judges are ordered left to right: Qwen3-4B, Qwen3-8B, Qwen3-14B, Llama-3.1-8B.}
\label{fig:routing}
\end{figure*}

\subsection{Single-mode vs.\ joint two-mode routing}
\label{sec:mode1_vs_joint}

Table~\ref{tab:mode1_vs_joint} isolates the contribution of adaptive retrieval by comparing Mode~1 Only against the full joint framework. On TriviaQA, Mode~1 already achieves high coverage for strong judges (87\% for Qwen3-14B), so retrieval provides a modest but consistent gain. The effect is more pronounced: on NQ-Open, coverage for Qwen3-8B increases from 7\% to 82\%; on HotpotQA, coverage for Llama-8B increases from 40\% to 85\%. Similarly, the weakest judge (Qwen3-4B) benefits substantially on PopQA, with coverage increasing from 4\% to 100\%.

\begin{table*}[!ht]
\centering
\caption{Coverage at $\alpha\!=\!0.20$: Mode~1 Only vs.\ Joint (two-threshold routing with retrieval).}
\label{tab:mode1_vs_joint}
\resizebox{\textwidth}{!}{%
\footnotesize
\begin{tabular}{ll|cc|cc|cc|cc}
\toprule
 & & \multicolumn{2}{c|}{TriviaQA} & \multicolumn{2}{c|}{NQ-Open} & \multicolumn{2}{c|}{PopQA} & \multicolumn{2}{c}{HotpotQA} \\
Cand. & Judge & M1 Only & Joint & M1 Only & Joint & M1 Only & Joint & M1 Only & Joint \\
\midrule
\multirow{4}{*}{\rotatebox{90}{\scriptsize Qwen-8B}}
 & Qwen-4B   & .47 & .91 & .01 & .14 & .04 & 1.0 & .00 & .14 \\
 & Qwen-8B   & .59 & 1.0 & .07 & .82 & .03 & 1.0 & .05 & .74 \\
 & Qwen-14B  & .87 & 1.0 & .14 & .81 & .15 & 1.0 & .35 & .86 \\
 & Llama-8B  & .77 & 1.0 & .22 & .83 & .21 & 1.0 & .40 & .85 \\
\midrule
\multirow{4}{*}{\rotatebox{90}{\scriptsize Llama-70B}}
 & Qwen-4B   & .67 & .98 & .07 & .32 & .05 & .35 & .09 & .34 \\
 & Qwen-8B   & .82 & .99 & .22 & .47 & .08 & .48 & .27 & .47 \\
 & Qwen-14B  & .84 & 1.0 & .12 & .51 & .15 & .77 & .24 & .53 \\
 & Llama-8B  & .48 & .86 & .02 & .32 & .12 & .77 & .02 & .27 \\
\bottomrule
\end{tabular}%
}
\end{table*}
\subsection{Robustness to retrieval non-stationarity}
\label{sec:drift}
Web search results change over time, which \S\ref{sec:guarantee} identified as the main practical threat to the determinism the guarantee assumes for Mode~2. We measure the effect directly on the main configuration (Qwen3-8B candidate, Qwen3-8B judge, LLM evaluator, $k\!=\!3$). We re-query the same 2{,}000 questions on TriviaQA and HotpotQA roughly three months after the original snapshot and recompute the Mode~2 verdicts on the fresh evidence. The retrieved pages turn over substantially. The mean per-item Jaccard overlap between the original ($\tau_0$) and re-queried ($\tau_1$) ranked top-$3$ URL lists is $0.30$ for TriviaQA and $0.24$ for HotpotQA. When ignoring the retrieval rank and considering only URL membership, the corresponding position-blind top-$3$ overlaps are $0.47$ and $0.40$, respectively.

To evaluate whether the calibrated thresholds remain valid after retrieval drift, we calibrate $(\hat{t}_1, \hat{t}_2)$ on the original data and apply them to test instances evaluated with the updated Mode~2 verdicts (while keeping Mode~1 unchanged). We repeat this evaluation over 100 random 50/50 splits. Table~\ref{tab:drift} summarizes the results. Across all settings, the empirical FDR remains below the target level~$\alpha$. Similarly, coverage remains largely stable after redeployment. This indicates that drift changes which pages are returned more than the overall quality of the evidence. Therefore, the calibrated thresholds remain effective after redeployment.

\begin{table*}[!ht]
\centering
\footnotesize
\caption{FDR and coverage under retrieval drift from the original
($\tau_0$) to the re-queried snapshot ($\tau_1$, $\sim$3 months later).
$\Delta=\tau_1-\tau_0$.}
\label{tab:drift}
\begin{tabular*}{\textwidth}{@{\extracolsep{\fill}}c|ccc|ccc|ccc|ccc@{}}
\toprule
 & \multicolumn{6}{c|}{TriviaQA} & \multicolumn{6}{c}{HotpotQA} \\
 & \multicolumn{3}{c|}{FDR} & \multicolumn{3}{c|}{Coverage} & \multicolumn{3}{c|}{FDR} & \multicolumn{3}{c}{Coverage} \\
$\alpha$ & $\tau_0$ & $\tau_1$ & $\Delta$ & $\tau_0$ & $\tau_1$ & $\Delta$ & $\tau_0$ & $\tau_1$ & $\Delta$ & $\tau_0$ & $\tau_1$ & $\Delta$ \\
\midrule
.05 & .03 & .04 & .00 & .12 & .17 & $+$.05 & .00 & .00 & .00 & .00 & .00 & .00 \\
.10 & .08 & .09 & .00 & .63 & .68 & $+$.06 & .05 & .03 & $-$.01 & .04 & .02 & $-$.02 \\
.15 & .14 & .13 & .00 & .91 & .94 & $+$.04 & .13 & .13 & .00 & .49 & .49 & .00 \\
.20 & .16 & .15 & $-$.01 & 1.0 & 1.0 & .00 & .18 & .18 & .00 & .74 & .72 & $-$.02 \\
.25 & .16 & .15 & $-$.01 & 1.0 & 1.0 & .00 & .23 & .23 & .00 & .96 & .96 & .00 \\
\bottomrule
\end{tabular*}
\end{table*}

\section{Conclusion}
We introduced a risk-controlled selective evaluation framework for LLM-as-a-judge with finite-sample guarantees on the reliability of accepted verdicts. The framework uses a calibrated two-mode policy that routes uncertain instances to a second evaluation mode and achieves higher coverage than single-mode abstention. More importantly, the Clopper--Pearson FDR guarantee remains valid under the joint two-threshold policy without requiring additional distributional assumptions. Across all tested configurations, the observed FDR remains at or below the specified target. Future work will consider graded and subjective evaluation settings, as well as black-box judges without access to logits.

\section*{Acknowledgments}
We acknowledge the support of the Natural Sciences and Engineering Research Council of Canada (NSERC), the Canada Foundation for Innovation (CFI), and Research Nova Scotia. Advanced computing resources are provided by ACENET, the regional partner in Atlantic Canada, and the Digital Research Alliance of Canada.

\bibliography{colm2026_conference}
\bibliographystyle{colm2026_conference}

\appendix

\appendix

\section{Proofs}
\label{app:proofs}

\subsection{Clopper--Pearson UCB Validity}
\label{app:proof_ucb}

We show that $\hat{R}^{\mathrm{upper}}_{1-\delta}(t)$ in Eq.~\eqref{eq:ucb} satisfies $\Pr(R(t) \leq \hat{R}^{\mathrm{upper}}_{1-\delta}(t)) \geq 1 - \delta$.

\begin{proof}
Under the Bernoulli structure, the observed failure count among $\hat{m}_{\mathrm{cal}}(t)$ selected instances follows $\hat{w}_{\mathrm{cal}}(t) \sim \mathrm{Binomial}(\hat{m}_{\mathrm{cal}}(t),\, R(t))$. Define the CDF of the empirical failure rate:
\begin{equation}
  D(r \mid R(t)) = \Pr\!\left(\hat{R}(t) \leq r \mid R(t)\right),
\end{equation}
where $\hat{R}(t) = \hat{w}_{\mathrm{cal}}(t) / \hat{m}_{\mathrm{cal}}(t)$. By definition:
\begin{equation}
  \hat{R}^{\mathrm{upper}}_{1-\delta}(t) = \sup\!\left\{R \in [0,1] : D(\hat{r}_{\mathrm{cal}}(t) \mid R) \geq \delta\right\}.
\end{equation}
Since $D(r \mid R)$ is monotonically decreasing in $R$ (higher true rate makes low empirical rates less likely), $R(t) > \hat{R}^{\mathrm{upper}}_{1-\delta}(t)$ implies $D(\hat{r}_{\mathrm{cal}}(t) \mid R(t)) < \delta$. Therefore:
\begin{equation}
  \Pr\!\left(R(t) > \hat{R}^{\mathrm{upper}}_{1-\delta}(t)\right)
  \leq \Pr\!\left(D(\hat{r}_{\mathrm{cal}}(t) \mid R(t)) < \delta\right)
  \leq \delta,
\end{equation}
where the last step uses the super-uniformity property: $\Pr(D(\hat{r}_{\mathrm{cal}}(t) \mid R(t)) < \delta) \leq \delta$ for discrete distributions. Hence $\Pr(R(t) \leq \hat{R}^{\mathrm{upper}}_{1-\delta}(t)) \geq 1 - \delta$.
\end{proof}

\subsection{Bernoulli Structure under Two-Mode Routing}
\label{app:proof_routing}

\begin{proof}[Proof of Lemma~\ref{lem:routing}]
For fixed thresholds $(t_1, t_2)$, we define the selection indicator:
\begin{equation}
  I_i := \ind\!\left\{U_1^{(i)} \!\leq\! t_1\right\} + \ind\!\left\{U_1^{(i)} \!>\! t_1 \;\text{and}\; U_2^{(i)} \!\leq\! t_2\right\}.
\end{equation}
Both $U_1^{(i)} = U(\mathcal{J}(\mathbf{x}_i, \hat{y}_i))$ and $U_2^{(i)} = U(\mathcal{J}(\mathbf{x}_i, \hat{y}_i, \mathcal{E}_i))$ are deterministic functions of $(\mathbf{x}_i, \hat{y}_i)$: the judge's computation and the retrieval process $\mathcal{E}_i$ are both determined by the input. Hence $I_i$ is a deterministic function of $(\mathbf{x}_i, \hat{y}_i)$.

The correctness of the selected verdict is:
\begin{align}
  c_i &= \ind\!\left\{U_1^{(i)} \!\leq\! t_1\right\} c_1^{(i)} \nonumber \\
  &\;+ \ind\!\left\{U_1^{(i)} \!>\! t_1 \,\wedge\, U_2^{(i)} \!\leq\! t_2\right\} c_2^{(i)},
\end{align}
which is a deterministic function of $(\mathbf{x}_i, \hat{y}_i, y^*_i)$, since the routing decision depends on $(\mathbf{x}_i, \hat{y}_i)$ and correctness additionally depends on $y^*_i$.

Since $(\mathbf{x}_i, \hat{y}_i, y^*_i)$ are i.i.d.\ from $\mathcal{D}$:
\begin{enumerate}
  \item \textbf{Independence}: The selection events $\{I_i = 1\}$ depend on disjoint instances, so they are mutually independent. Conditioning on selection preserves independence across instances.
  \item \textbf{Identical distribution}: Each selected instance satisfies $(\mathbf{x}_i, \hat{y}_i, y^*_i) \sim \mathcal{D}_{t_1,t_2}$, where $\mathcal{D}_{t_1,t_2} := \mathcal{D} \mid I = 1$ is the conditional distribution.
\end{enumerate}
Note that $I_i \in \{0,1\}$ because the two indicator terms are mutually exclusive: the second requires $U_1^{(i)} > t_1$, which precludes the first.

The failure indicators $W_i = 1 - c_i$ over the selected subset are therefore i.i.d.\ $\mathrm{Bernoulli}(R(t_1, t_2))$, where $R(t_1, t_2) = \mathbb{E}[W \mid I = 1]$.
\end{proof}

\paragraph{Scope and limitations.}
The guarantee is tied to the retrieval snapshot used during calibration. Because retrieval relies on a non-stationary web index, a query may return different evidence at deployment than it did at calibration. Such retrieval drift affects the guarantee only if it changes the conditional error rate of Mode~2 verdicts. As long as this error rate remains stable, the calibrated thresholds continue to transfer. We evaluate this scenario in Section~\ref{sec:drift} by re-querying the web several months after calibration and find that FDR control is preserved. If the quality of the retrieved evidence changes substantially, recalibration restores the guarantee. Likewise, any change to the task, candidate model, or judge model violates the exchangeability assumption and requires recalibration.

\section{Experimental Details}
\label{app:expdetails}

\subsection{Notation}
\label{app:notation}

Table~\ref{tab:notation_full} provides a complete reference of all symbols.

\begin{table}[!ht]
\small
\centering
\caption{Notation reference.}
\label{tab:notation_full}
\begin{tabular}{@{}p{1.7cm}p{5.2cm}@{}}
\toprule
\textbf{Symbol} & \textbf{Description} \\
\midrule
$\mathbf{x},\; \hat{y},\; y^*$ & Question, candidate answer, ground truth \\
$F,\; \mathcal{J}$ & Examinee LLM, judge LLM \\
$v \in \{0,1\}$ & Evaluation verdict \\
$e^* \in \{0,1\}$ & Ground-truth evaluation label \\
$\mathcal{A}(v, e^*)$ & Admission function: $\ind[v\!=\!e^*]$ \\
$\mathcal{A}_{\mathrm{ref}},\; \lambda_A$ & Relevance function and threshold \\
$W$ & Failure indicator: $\ind[v \neq e^*]$ \\
$U,\; U_1,\; U_2$ & Uncertainty (general / Mode 1 / Mode 2) \\
$v_1,\; v_2$ & Verdict from Mode 1 / Mode 2 \\
$c_1,\; c_2$ & Correctness of Mode 1 / Mode 2 \\
$\mathcal{E}$ & Retrieved external evidence \\
$\mathcal{D}_{\mathrm{cal}}$ & Calibration set ($N$ instances) \\
$\hat{m}_{\mathrm{cal}}$ & Selection size on calibration set \\
$\hat{w}_{\mathrm{cal}}$ & Error count on calibration set \\
$\hat{r}_{\mathrm{cal}}$ & Empirical FDR \\
$R(t)$ & True FDR \\
$\hat{R}^{\mathrm{upper}}_{1-\delta}$ & Clopper--Pearson UCB \\
$\hat{t}_1,\; \hat{t}_2$ & Calibrated thresholds (Mode 1, 2) \\
$\alpha$ & User-specified risk level \\
$\delta$ & Significance level (default $0.05$) \\
$K$ & Sampling size for black-box UQ \\
$\mathrm{Coverage}(t)$ & Coverage: $|\{i : U_i \leq t\}| / N_{\mathrm{test}}$ \\
\bottomrule
\end{tabular}
\end{table}

\subsection{Implementation}
\label{app:implementation}

\paragraph{Candidate generation.}
We generate candidate answers using greedy decoding for both Qwen3-8B and LLaMA-3.1-70B. For Qwen3-8B, thinking mode is disabled (\texttt{enable\_thinking=False}) to prevent the \texttt{<think>} block from consuming the generation budget.

\paragraph{Judge inference.}
For predictive entropy (PE), we perform a single forward pass at the verdict token position and extract the log-probabilities for the \texttt{True} and \texttt{False} tokens. PE is computed as binary entropy: $-p_{\text{true}} \log p_{\text{true}} - (1 - p_{\text{true}}) \log (1 - p_{\text{true}})$. This requires no additional inference beyond the greedy decoding pass used to produce the verdict.

\paragraph{Retrieval.}
We query the Serper API with the original question and retrieve the top-$k$ results. Each result contains a title, text snippet, and source URL. Results are concatenated into a single evidence passage appended to the Mode~2 judge prompt.

\paragraph{Calibration.}
We run 100 random 50/50 calibration/test splits. For each split, we perform a grid search over all unique uncertainty values on the calibration set to find the threshold pair $(\hat{t}_1, \hat{t}_2)$ that maximizes correctly accepted instances subject to the Clopper--Pearson UCB constraint $\hat{R}^{\mathrm{upper}}_{1-\delta}(t_1, t_2) \leq \alpha$. The grid search uses cumulative sums for $O(|\mathcal{T}_1| \cdot |\mathcal{T}_2|)$ complexity.

\subsection{Prompt Templates}
\label{app:prompts}

\subsubsection{Judge Prompts}
\paragraph{Mode~1 (parametric evaluation).}
The judge receives a system prompt and a user message containing the question and candidate answer. The prompt includes a one-shot example to establish the output format.

\begin{tcolorbox}[colback=gray!5, colframe=gray!60, title={\small Mode 1: Judge System Prompt}, fonttitle=\bfseries\small, fontupper=\small]
You are a helpful assistant acting as an impartial judge. You will be given a Question and a Candidate Answer. Your task is to judge whether the Candidate Answer is correct using your own knowledge. If the Candidate Answer is correct, choose `True', otherwise, choose `False'. Provide a brief explanation for your decision.

\medskip
Provide your response in the following format:\\
\texttt{Decision: [True/False]}\\
\texttt{Explanation: [Your brief explanation]}

\medskip
\textit{Example:}\\
Question: What is the capital of France?\\
Candidate Answer: Paris is the capital of France.\\
Decision: True\\
Explanation: Paris is indeed the capital of France.
\end{tcolorbox}

\begin{tcolorbox}[colback=blue!3, colframe=blue!40, title={\small Mode 1: User Message}, fonttitle=\bfseries\small, fontupper=\small]
Question: \texttt{\{question\}}

Candidate Answer: \texttt{\{candidate\_answer\}}
\end{tcolorbox}

\paragraph{Mode~2 (retrieval-augmented evaluation).}
The judge additionally receives retrieved web search results. Each result is formatted as: \texttt{[i] Title / Snippet / Source: URL}.

\begin{tcolorbox}[colback=gray!5, colframe=gray!60, title={\small Mode 2: Judge System Prompt}, fonttitle=\bfseries\small, fontupper=\small]
You are a helpful assistant acting as an impartial judge. You will be given a Question, a Candidate Answer, and Web Search Results. Your task is to judge whether the Candidate Answer is correct using the provided search results. If the search results support the Candidate Answer, choose `True'. If the search results contradict it, choose `False'. If the search results are inconclusive, use your best judgment. Provide a brief explanation for your decision.

\medskip
Provide your response in the following format:\\
\texttt{Decision: [True/False]}\\
\texttt{Explanation: [Your brief explanation]}
\end{tcolorbox}

\begin{tcolorbox}[colback=blue!3, colframe=blue!40, title={\small Mode 2: User Message}, fonttitle=\bfseries\small, fontupper=\small]
Question: \texttt{\{question\}}

Candidate Answer: \texttt{\{candidate\_answer\}}

Web Search Results: \texttt{\{evidence\}}
\end{tcolorbox}

\subsection{LLM Evaluator Prompt}

The LLM evaluator (Qwen2.5-7B-Instruct) determines whether the candidate answer is semantically equivalent to any reference answer~\citep{badshah2025ref}, serving as the admission function for computing the ground-truth label~$e^*$.

\begin{tcolorbox}[colback=green!3, colframe=green!50!black, title={\small LLM Evaluator (Admission Function)}, fonttitle=\bfseries\small, fontupper=\small]
You are a helpful assistant acting as an impartial evaluator. Given a question, a candidate answer, and reference answers, determine if the candidate answer is correct.

\medskip
The candidate answer is correct if it conveys the same meaning as ANY of the reference answers, even if the candidate answer is longer, contains additional explanation, uses synonyms, abbreviations, or alternate names. Focus on whether the final answer or conclusion matches a reference, not on surface form.

\medskip
Question: \texttt{\{question\}}

Candidate Answer: \texttt{\{candidate\_answer\}}

Reference Answers: \texttt{\{references\}}

\medskip
Is the candidate answer correct? Respond with ONLY ``True'' or ``False''.
\end{tcolorbox}

\section{Additional Results}
\label{app:additional_results}
This appendix provides additional results that complement the main paper, including raw judge accuracy, an analysis of the joint framework across different risk levels, and per-dataset coverage and FDR curves.

\subsection{Judge Verdict Accuracy}
\label{app:accuracy}
Table~\ref{tab:verdict_accuracy} reports the raw verdict accuracy of each judge in Mode~1 (parametric) and Mode~2 (retrieval-augmented) before any calibration or selective prediction. Retrieval consistently improves accuracy across all configurations, with the largest gains on PopQA (up to 44 percentage points), showing that retrieval provides useful evidence.

\begin{table*}[!ht]
\centering
\small
\caption{Judge verdict accuracy (\%) before calibration. Mode~1 uses parametric knowledge only whereas Mode~2 uses retrieval-augmented evaluation ($k\!=\!3$). LLM evaluator as admission function.}
\label{tab:verdict_accuracy}
\setlength{\tabcolsep}{11.5pt}
\begin{tabular}{ll|cc|cc|cc|cc}
\toprule
 & & \multicolumn{2}{c|}{TriviaQA} & \multicolumn{2}{c|}{NQ-Open} & \multicolumn{2}{c|}{PopQA} & \multicolumn{2}{c}{HotpotQA} \\
Cand. & Judge & M1 & M2 & M1 & M2 & M1 & M2 & M1 & M2 \\
\midrule
\multirow{4}{*}{\rotatebox{90}{\scriptsize Qwen-8B}}
 & Qwen-4B  & 68 & 79 & 52 & 71 & 47 & 84 & 54 & 69 \\
 & Qwen-8B  & 71 & 84 & 54 & 79 & 44 & 88 & 58 & 76 \\
 & Qwen-14B & 78 & 86 & 61 & 77 & 65 & 88 & 67 & 78 \\
 & Llama-8B & 77 & 82 & 63 & 78 & 68 & 88 & 68 & 78 \\
\midrule
\multirow{4}{*}{\rotatebox{90}{\scriptsize Llama-70B}}
 & Qwen-4B  & 75 & 81 & 62 & 69 & 60 & 75 & 63 & 68 \\
 & Qwen-8B  & 78 & 81 & 63 & 71 & 61 & 77 & 65 & 69 \\
 & Qwen-14B & 78 & 82 & 64 & 72 & 66 & 77 & 67 & 72 \\
 & Llama-8B & 72 & 78 & 63 & 70 & 64 & 77 & 64 & 68 \\
\bottomrule
\end{tabular}
\end{table*}

\subsection{FDR and Coverage Across Risk Levels}
\label{app:full_results}
Table~\ref{tab:full_results} reports FDR and coverage for the joint two-mode framework (two-threshold routing with retrieval) at all evaluated $\alpha$ levels for the Qwen-8B candidate across all four datasets and all four judges (LLM evaluator, PE, $k\!=\!3$).

\begin{table*}[h]
\centering
\small
\caption{Full results for the joint two-mode framework at all $\alpha$ levels across all four datasets. Candidate: Qwen3-8B, LLM evaluator, PE, $k\!=\!3$.}
\label{tab:full_results}
\resizebox{\textwidth}{!}{%
\begin{tabular}{ll|cccccccccccc}
\toprule
 & & \multicolumn{12}{c}{$\alpha$} \\
Judge & Metric & .05 & .10 & .15 & .20 & .25 & .30 & .35 & .40 & .45 & .50 & .55 & .60 \\
\midrule
\multicolumn{14}{c}{\textit{TriviaQA}} \\
\midrule
\multirow{2}{*}{Qwen-4B}
 & FDR & .03 & .08 & .13 & .18 & .21 & .21 & .21 & .21 & .21 & .21 & .21 & .21 \\
 & Cov. & .03 & .39 & .66 & .91 & 1.0 & 1.0 & 1.0 & 1.0 & 1.0 & 1.0 & 1.0 & 1.0 \\
\multirow{2}{*}{Qwen-8B}
 & FDR & .03 & .08 & .14 & .16 & .16 & .16 & .16 & .16 & .16 & .16 & .16 & .16 \\
 & Cov. & .12 & .63 & .91 & 1.0 & 1.0 & 1.0 & 1.0 & 1.0 & 1.0 & 1.0 & 1.0 & 1.0 \\
\multirow{2}{*}{Qwen-14B}
 & FDR & .04 & .09 & .13 & .14 & .14 & .14 & .14 & .14 & .14 & .14 & .14 & .14 \\
 & Cov. & .21 & .78 & .98 & 1.0 & 1.0 & 1.0 & 1.0 & 1.0 & 1.0 & 1.0 & 1.0 & 1.0 \\
\multirow{2}{*}{Llama-8B}
 & FDR & .03 & .09 & .14 & .18 & .18 & .18 & .18 & .18 & .18 & .18 & .18 & .18 \\
 & Cov. & .13 & .61 & .86 & 1.0 & 1.0 & 1.0 & 1.0 & 1.0 & 1.0 & 1.0 & 1.0 & 1.0 \\
\midrule
\multicolumn{14}{c}{\textit{NQ-Open}} \\
\midrule
\multirow{2}{*}{Qwen-4B}
 & FDR & .00 & .05 & .11 & .17 & .22 & .27 & .29 & .29 & .29 & .29 & .29 & .29 \\
 & Cov. & .00 & .03 & .08 & .14 & .28 & .93 & 1.0 & 1.0 & 1.0 & 1.0 & 1.0 & 1.0 \\
\multirow{2}{*}{Qwen-8B}
 & FDR & .00 & .05 & .13 & .18 & .22 & .22 & .22 & .22 & .22 & .22 & .22 & .22 \\
 & Cov. & .00 & .02 & .24 & .82 & 1.0 & 1.0 & 1.0 & 1.0 & 1.0 & 1.0 & 1.0 & 1.0 \\
\multirow{2}{*}{Qwen-14B}
 & FDR & .00 & .05 & .13 & .18 & .22 & .23 & .23 & .23 & .23 & .23 & .23 & .23 \\
 & Cov. & .00 & .05 & .56 & .81 & .99 & 1.0 & 1.0 & 1.0 & 1.0 & 1.0 & 1.0 & 1.0 \\
\multirow{2}{*}{Llama-8B}
 & FDR & .00 & .09 & .13 & .18 & .22 & .22 & .22 & .22 & .22 & .22 & .22 & .22 \\
 & Cov. & .00 & .20 & .61 & .83 & 1.0 & 1.0 & 1.0 & 1.0 & 1.0 & 1.0 & 1.0 & 1.0 \\
\midrule
\multicolumn{14}{c}{\textit{PopQA}} \\
\midrule
\multirow{2}{*}{Qwen-4B}
 & FDR & .00 & .01 & .13 & .16 & .16 & .16 & .16 & .16 & .16 & .16 & .16 & .16 \\
 & Cov. & .00 & .01 & .91 & 1.0 & 1.0 & 1.0 & 1.0 & 1.0 & 1.0 & 1.0 & 1.0 & 1.0 \\
\multirow{2}{*}{Qwen-8B}
 & FDR & .00 & .09 & .12 & .12 & .12 & .12 & .12 & .12 & .12 & .12 & .12 & .12 \\
 & Cov. & .00 & .87 & 1.0 & 1.0 & 1.0 & 1.0 & 1.0 & 1.0 & 1.0 & 1.0 & 1.0 & 1.0 \\
\multirow{2}{*}{Qwen-14B}
 & FDR & .02 & .09 & .12 & .12 & .12 & .12 & .12 & .12 & .12 & .12 & .12 & .12 \\
 & Cov. & .19 & .89 & 1.0 & 1.0 & 1.0 & 1.0 & 1.0 & 1.0 & 1.0 & 1.0 & 1.0 & 1.0 \\
\multirow{2}{*}{Llama-8B}
 & FDR & .00 & .09 & .12 & .13 & .13 & .13 & .13 & .13 & .13 & .13 & .13 & .13 \\
 & Cov. & .02 & .88 & 1.0 & 1.0 & 1.0 & 1.0 & 1.0 & 1.0 & 1.0 & 1.0 & 1.0 & 1.0 \\
\midrule
\multicolumn{14}{c}{\textit{HotpotQA}} \\
\midrule
\multirow{2}{*}{Qwen-4B}
 & FDR & .00 & .05 & .10 & .15 & .21 & .27 & .31 & .31 & .31 & .31 & .31 & .31 \\
 & Cov. & .00 & .04 & .09 & .14 & .30 & .58 & .86 & .86 & .86 & .86 & .86 & .86 \\
\multirow{2}{*}{Qwen-8B}
 & FDR & .00 & .05 & .13 & .18 & .23 & .24 & .24 & .24 & .24 & .24 & .24 & .24 \\
 & Cov. & .00 & .04 & .49 & .74 & .96 & 1.0 & 1.0 & 1.0 & 1.0 & 1.0 & 1.0 & 1.0 \\
\multirow{2}{*}{Qwen-14B}
 & FDR & .02 & .08 & .13 & .18 & .22 & .22 & .22 & .22 & .22 & .22 & .22 & .22 \\
 & Cov. & .02 & .26 & .64 & .86 & 1.0 & 1.0 & 1.0 & 1.0 & 1.0 & 1.0 & 1.0 & 1.0 \\
\multirow{2}{*}{Llama-8B}
 & FDR & .01 & .08 & .13 & .18 & .22 & .22 & .22 & .22 & .22 & .22 & .22 & .22 \\
 & Cov. & .01 & .30 & .65 & .85 & 1.0 & 1.0 & 1.0 & 1.0 & 1.0 & 1.0 & 1.0 & 1.0 \\
\bottomrule
\end{tabular}%
}
\end{table*}

\subsubsection{Coverage Curves}
Figure~\ref{fig:coverage_all} shows coverage as a function of~$\alpha$ across all 32 configurations.

\begin{figure*}[!ht]
\centering
\includegraphics[width=\textwidth]{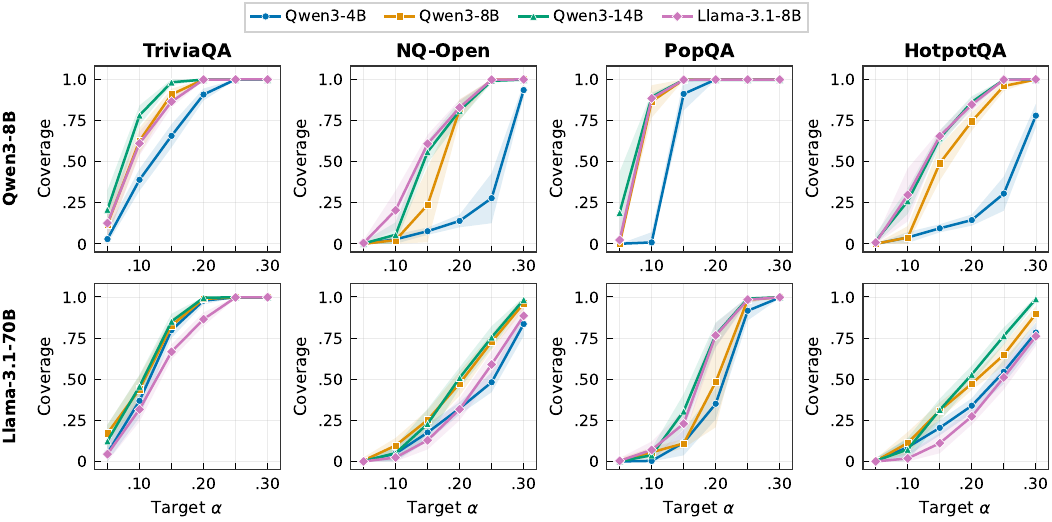}
\caption{Coverage (mean$\pm$std over 100 splits) across all 32 configurations. Rows: candidates; columns: datasets.}
\label{fig:coverage_all}
\end{figure*}

\subsubsection{Per-Dataset FDR}
Figures~\ref{fig:fdr_triviaqa}--\ref{fig:fdr_hotpotqa} provide per-dataset FDR plots for detailed examination.

\begin{figure*}[!ht]
\centering
\includegraphics[width=0.85\textwidth]{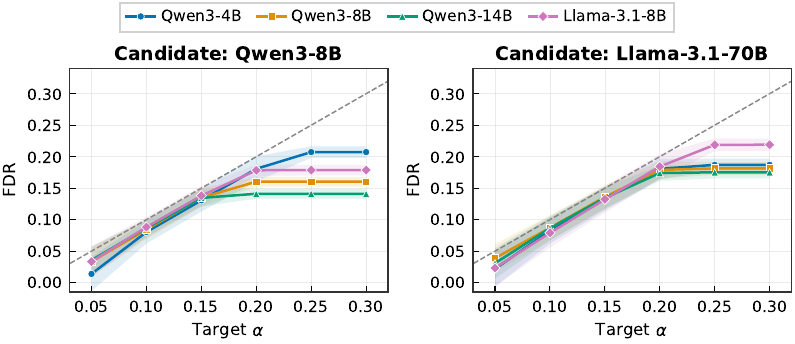}
\caption{FDR control on TriviaQA.}
\label{fig:fdr_triviaqa}
\end{figure*}

\begin{figure*}[!ht]
\centering
\includegraphics[width=0.85\textwidth]{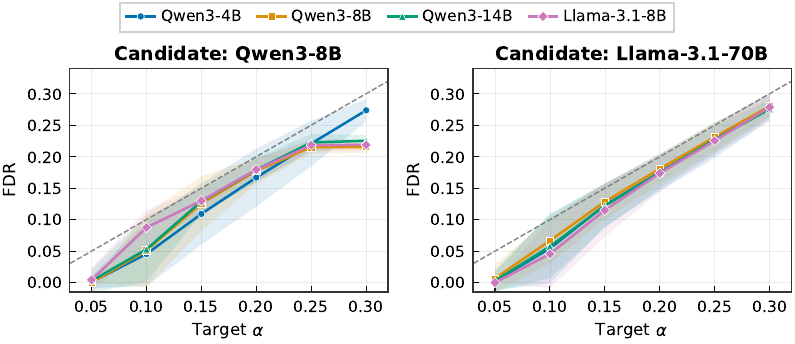}
\caption{FDR control on NQ-Open.}
\label{fig:fdr_nqopen}
\end{figure*}

\begin{figure*}[!ht]
\centering
\includegraphics[width=0.85\textwidth]{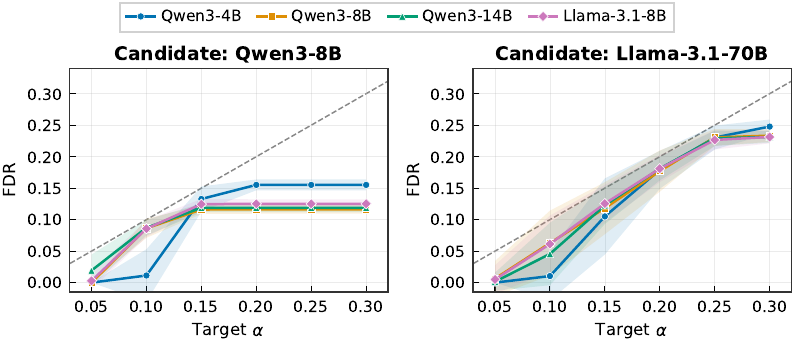}
\caption{FDR control on PopQA.}
\label{fig:fdr_popqa}
\end{figure*}

\begin{figure*}[!ht]
\centering
\includegraphics[width=0.85\textwidth]{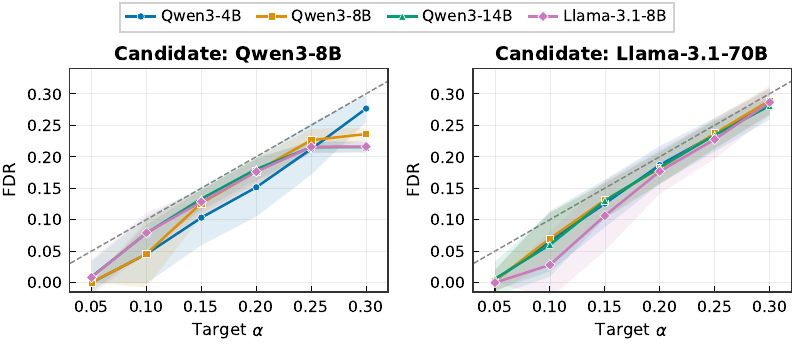}
\caption{FDR control on HotpotQA.}
\label{fig:fdr_hotpotqa}
\end{figure*}

\section{Ablations and Robustness}
\label{app:ablations}
We evaluate the framework across several dimensions. This includes pipeline design choices, the calibration procedure, the uncertainty score, and the choice of no-retrieval baseline. In all cases, the finite-sample FDR guarantee is preserved, with conservative choices reducing coverage rather than compromising risk control.

\subsection{Sensitivity Analyses}
\label{sec:sensitivity}
\label{app:calibsens}

\paragraph{Correctness criterion (admission function).}
Table~\ref{tab:admission_sensitivity} compares three admission functions across three risk levels on a representative pair (Qwen3-8B candidate, Qwen3-14B judge). The LLM evaluator consistently achieves the highest coverage, with the gap most pronounced on knowledge-intensive benchmarks: on NQ-Open at $\alpha\!=\!0.15$, it reaches 56\% while both exact match and token~F1 remain at 0\%. This gap arises because string-matching criteria penalize semantically correct but surface-different answers, inflating the judge's apparent error rate. FDR control holds under all three criteria.

\begin{table*}[h]
\centering
\caption{Sensitivity to admission function. Candidate: Qwen3-8B, Judge: Qwen3-14B, UQ: PE (white-box). FDR and coverage at $\alpha \in \{.15, .20, .25\}$.}
\label{tab:admission_sensitivity}
\resizebox{\textwidth}{!}{%
\footnotesize
\begin{tabular}{l|ccc|ccc|ccc|ccc|ccc|ccc}
\toprule
 & \multicolumn{6}{c|}{Exact Match} & \multicolumn{6}{c|}{Token F1} & \multicolumn{6}{c}{LLM Evaluator} \\
 & \multicolumn{3}{c|}{FDR} & \multicolumn{3}{c|}{Coverage} & \multicolumn{3}{c|}{FDR} & \multicolumn{3}{c|}{Coverage} & \multicolumn{3}{c|}{FDR} & \multicolumn{3}{c}{Coverage} \\
Dataset & .15 & .20 & .25 & .15 & .20 & .25 & .15 & .20 & .25 & .15 & .20 & .25 & .15 & .20 & .25 & .15 & .20 & .25 \\
\midrule
TriviaQA & .13 & .14 & .14 & .97 & 1.0 & 1.0 & .13 & .18 & .18 & .72 & .98 & 1.0 & .13 & .14 & .14 & .98 & 1.0 & 1.0 \\
NQ-Open  & .00 & .05 & .22 & .00 & .11 & .69 & .00 & .00 & .18 & .00 & .00 & .45 & .13 & .18 & .22 & .56 & .81 & .99 \\
PopQA    & .12 & .12 & .12 & 1.0 & 1.0 & 1.0 & .13 & .13 & .13 & .99 & 1.0 & 1.0 & .12 & .12 & .12 & 1.0 & 1.0 & 1.0 \\
HotpotQA & .04 & .18 & .23 & .03 & .57 & .86 & .04 & .17 & .23 & .04 & .47 & .83 & .13 & .18 & .22 & .64 & .86 & 1.0 \\
\bottomrule
\end{tabular}%
}
\end{table*}

\paragraph{Retrieval depth.}
Table~\ref{tab:retrieval_depth} ablates the number of retrieved passages ($k \in \{3, 6, 9\}$) on a representative pair (Llama-70B candidate, Qwen3-8B judge). Coverage and FDR are stable across retrieval depths, with differences within 1--2 percentage points. $k\!=\!3$ already provides sufficient evidence.

\begin{table}[h]
\centering
\small
\caption{Retrieval depth ablation at $\alpha\!=\!0.20$. Candidate: Llama-70B, Judge: Qwen3-8B, LLM evaluator.}
\label{tab:retrieval_depth}
\setlength{\tabcolsep}{11.5pt}
\begin{tabular}{l|cc|cc|cc}
\toprule
 & \multicolumn{2}{c|}{$k=3$} & \multicolumn{2}{c|}{$k=6$} & \multicolumn{2}{c}{$k=9$} \\
Dataset & FDR & Cov. & FDR & Cov. & FDR & Cov. \\
\midrule
TriviaQA & .18 & .99 & .17 & 1.0 & .17 & .99 \\
NQ-Open  & .18 & .47 & .18 & .46 & .18 & .46 \\
PopQA    & .18 & .48 & .18 & .41 & .17 & .36 \\
HotpotQA & .18 & .47 & .18 & .46 & .18 & .46 \\
\bottomrule
\end{tabular}
\end{table}

\paragraph{Confidence level $\delta$.} Table~\ref{tab:delta} varies the confidence level. Tightening $\delta$ from $0.05$ to $0.01$ trades coverage for confidence, with the cost concentrated at low~$\alpha$ and on the harder dataset; FDR stays controlled throughout.

\begin{table}[h]
\centering
\small
\caption{Sensitivity to the confidence level $\delta$ (coverage). Candidate: Qwen3-8B, Judge: Qwen3-8B, LLM evaluator, $k\!=\!3$, PE. FDR is controlled at the target~$\alpha$ under both settings.}
\label{tab:delta}
\begin{tabular*}{\textwidth}{@{\extracolsep{\fill}}c|ccccc|ccccc@{}}
\toprule
 & \multicolumn{5}{c|}{TriviaQA} & \multicolumn{5}{c}{HotpotQA} \\
$\delta~\backslash~\alpha$ & .05 & .10 & .15 & .20 & .25 & .05 & .10 & .15 & .20 & .25 \\
\midrule
0.05 & .12 & .63 & .91 & 1.0 & 1.0 & .00 & .04 & .49 & .74 & .96 \\
0.01 & .04 & .53 & .87 & 1.0 & 1.0 & .00 & .01 & .37 & .69 & .93 \\
\bottomrule
\end{tabular*}
\end{table}

\paragraph{Calibration-set size.} Table~\ref{tab:calibsize} sweeps the calibration fraction $n_{\mathrm{cal}}/n$ over the 2{,}000 items per dataset. Coverage on the easier TriviaQA is stable across the range, while on HotpotQA it degrades gracefully as the calibration set shrinks; FDR is controlled at every size.

\begin{table}[h]
\centering
\small
\caption{Sensitivity to calibration-set size (coverage). $n_{\mathrm{cal}}/n$ is the fraction of the 2{,}000 items used for calibration. Candidate: Qwen3-8B, Judge: Qwen3-8B, LLM evaluator, $k\!=\!3$, PE. FDR is controlled at the target~$\alpha$ for every size.}
\label{tab:calibsize}
\begin{tabular*}{\textwidth}{@{\extracolsep{\fill}}c|ccccc|ccccc@{}}
\toprule
 & \multicolumn{5}{c|}{TriviaQA} & \multicolumn{5}{c}{HotpotQA} \\
$n_{\mathrm{cal}}/n~\backslash~\alpha$ & .05 & .10 & .15 & .20 & .25 & .05 & .10 & .15 & .20 & .25 \\
\midrule
0.10 & .02 & .38 & .81 & .96 & .99 & .00 & .03 & .25 & .55 & .83 \\
0.25 & .08 & .54 & .89 & .99 & 1.0 & .00 & .04 & .41 & .70 & .92 \\
0.50 & .12 & .63 & .91 & 1.0 & 1.0 & .00 & .04 & .49 & .74 & .96 \\
0.75 & .15 & .66 & .92 & 1.0 & 1.0 & .00 & .04 & .53 & .77 & .98 \\
\bottomrule
\end{tabular*}
\end{table}

\subsection{Alternative Uncertainty Estimators}
\label{app:uqmenu}
The FDR guarantee depends only on the binary correctness indicator. As a result, the choice of uncertainty score affects coverage but not whether the target risk level is controlled. In addition to the white-box predictive entropy (PE) used in the main paper, we evaluate two supervised uncertainty scores derived from the judge's last-layer hidden state at the verdict position. Following~\citet{NEURIPS2024_9c20f16b}, we consider a small MLP (\emph{Probe}) and a \emph{LoRA+Prompt} adapter fine-tuned to predict verdict correctness.

Both models are trained using 5-fold cross-validation on the calibration set so that every calibration instance receives an out-of-sample uncertainty score before conformal calibration. Table~\ref{tab:uqmenu} reports the resulting coverage. All three uncertainty scores satisfy the target FDR level across every setting.

LoRA+Prompt consistently achieves the highest coverage, particularly on datasets where PE is less informative. For example, coverage increases from $0.04$ to $0.44$ on HotpotQA at $\alpha\!=\!0.10$ and from $0.00$ to $0.71$ on PopQA at $\alpha\!=\!0.05$. The Probe also improves over PE on more challenging datasets at moderate risk levels (e.g., NQ-Open at $\alpha\!=\!0.15$, $0.24 \rightarrow 0.42$), although it performs worse at the lowest risk levels.

\begin{table}[!ht]
\centering
\small
\caption{Coverage under three uncertainty scores: the white-box predictive entropy (PE) used in the main paper and two supervised scores, a \emph{Probe} and a \emph{LoRA+Prompt} adapter on the judge's hidden states. Candidate: Qwen3-8B, Judge: Qwen3-8B, LLM evaluator, $k\!=\!3$. FDR is controlled at the target~$\alpha$ for all three.}
\label{tab:uqmenu}
\setlength{\tabcolsep}{4pt}
\begin{tabular*}{\textwidth}{@{\extracolsep{\fill}}c|ccc|ccc|ccc|ccc@{}}
\toprule
 & \multicolumn{3}{c|}{TriviaQA} & \multicolumn{3}{c|}{NQ-Open} & \multicolumn{3}{c|}{PopQA} & \multicolumn{3}{c}{HotpotQA} \\
$\alpha$ & PE & Probe & LoRA & PE & Probe & LoRA & PE & Probe & LoRA & PE & Probe & LoRA \\
\midrule
.05 & .12 & .00 & .00 & .00 & .00 & .02 & .00 & .31 & .71 & .00 & .00 & .02 \\
.10 & .63 & .14 & .70 & .02 & .07 & .28 & .87 & .82 & .94 & .04 & .02 & .44 \\
.15 & .91 & .90 & .94 & .24 & .42 & .64 & 1.0 & 1.0 & 1.0 & .49 & .47 & .76 \\
.20 & 1.0 & 1.0 & 1.0 & .82 & .79 & .86 & 1.0 & 1.0 & 1.0 & .74 & .78 & .95 \\
.25 & 1.0 & 1.0 & 1.0 & 1.0 & 1.0 & 1.0 & 1.0 & 1.0 & 1.0 & .96 & .99 & 1.0 \\
\bottomrule
\end{tabular*}
\end{table}

\subsection{Baseline Comparison}
\label{app:baselines}
To isolate the contribution of retrieved evidence, Table~\ref{tab:baselines} compares our joint two-mode policy against three controls that keep the identical calibration but change what the accepted verdict is based on. \emph{Direct} accepts only Mode~1 (no retrieval). \emph{Self-eval} adds a second Mode~1 pass in which the judge is asked whether its own initial verdict is correct, using no new evidence. \emph{Empty} runs the Mode~2 prompt with an empty evidence slot. FDR is controlled for every variant. The joint policy dominates all three at nearly every operating point, and the gap is largest on the knowledge-intensive HotpotQA ($0.74$ versus at most $0.07$ for the no-retrieval controls at $\alpha\!=\!0.20$). A second evaluation pass without new evidence (Self-eval) barely helps, confirming that the coverage gains come from the retrieved evidence rather than from simply re-querying the judge.

\begin{table}[!ht]
\centering
\small
\caption{Baseline comparison (coverage). Candidate: Qwen3-8B, Judge: Qwen3-8B, LLM evaluator, $k\!=\!3$, PE. FDR is controlled at the target~$\alpha$ for every variant.}
\label{tab:baselines}
\begin{tabular*}{\textwidth}{@{\extracolsep{\fill}}l|ccccc|ccccc@{}}
\toprule
 & \multicolumn{5}{c|}{TriviaQA} & \multicolumn{5}{c}{HotpotQA} \\
Variant~$\backslash$~$\alpha$ & .05 & .10 & .15 & .20 & .25 & .05 & .10 & .15 & .20 & .25 \\
\midrule
Direct       & .03 & .28 & .47 & .59 & .74 & .00 & .00 & .01 & .05 & .17 \\
Self-eval    & .02 & .32 & .49 & .60 & .75 & .00 & .00 & .02 & .07 & .19 \\
Empty        & .04 & .38 & .63 & .83 & .99 & .00 & .09 & .32 & .57 & .84 \\
\midrule
Joint (ours) & .12 & .63 & .91 & 1.0 & 1.0 & .00 & .04 & .49 & .74 & .96 \\
\bottomrule
\end{tabular*}
\end{table}

\subsection{Conservative Joint Calibration}
\label{app:bonferroni}
We selected the deployed threshold pair $(\hat{t}_1, \hat{t}_2)$ using the calibration data rather than fixing them beforehand. Specifically, we choose the pair that maximizes coverage while satisfying $\hat{R}^{\mathrm{upper}}_{1-\delta} \leq \alpha$ (Eq.~\ref{eq:that_route}). Although the Clopper--Pearson bound holds for each candidate pair on its own, certifying the selected pair requires accounting for this search. The conservative variant addresses this directly by evaluating all $|\mathcal{T}_1|\,|\mathcal{T}_2|$ candidate pairs using the corrected confidence level $\delta'=\delta/(|\mathcal{T}_1|\,|\mathcal{T}_2|)$. By a union-bound argument, the guarantee then holds simultaneously for every candidate pair, implying that the selected pair satisfies $\Pr(R(\hat{t}_1,\hat{t}_2)\leq\alpha)\geq1-\delta$. In practice, however, we find that the simpler uncorrected selection already provides effective risk control (see Section~\ref{sec:results}). For this reason, we use the coverage-maximizing thresholds in our main experiments, while treating the conservative procedure as a theoretically valid alternative.

One subtlety is that the grid is data-dependent, since $\mathcal{T}_1$ and $\mathcal{T}_2$ are the observed uncertainty values, making both the candidate pairs and their count random. Because each mode contributes at most $N$ distinct thresholds, replacing $|\mathcal{T}_1|\,|\mathcal{T}_2|$ with the data-independent bound $N^2$ (equivalently, correcting at level $\delta/N^2$) makes the union bound hold over a fixed family and leaves the conclusion unchanged at negligible cost; pre-specifying the grid removes the dependence altogether.

Table~\ref{tab:bonferroni} compares the coverage of the main-paper calibration (\emph{Pointwise}, per-pair level $\delta$) with this \emph{Conservative} variant (grid-corrected level $\delta'$). FDR is controlled under both. The correction is nearly free on the easier datasets (TriviaQA, PopQA) but costs substantial coverage on the harder ones (NQ-Open, HotpotQA), where the empirical FDR sits close to~$\alpha$. Bonferroni is a worst-case union bound and adjacent grid thresholds select almost the same instances, so it is very loose. The truly attainable jointly-valid coverage lies between the two rows which is much closer to the pointwise figures.

\begin{table}[h]
\centering
\small
\caption{Conservative joint calibration (coverage). Per dataset, \emph{Point.}\ and \emph{Cons.}\ denote the pointwise and conservative coverage, and $\Delta = \text{Cons.} - \text{Point.}$ (computed on unrounded values). Candidate: Qwen3-8B, Judge: Qwen3-8B, LLM evaluator, $k\!=\!3$, PE. FDR is controlled at the target~$\alpha$ under both variants.}
\label{tab:bonferroni}
\setlength{\tabcolsep}{4pt}
\begin{tabular*}{\textwidth}{@{\extracolsep{\fill}}c|ccc|ccc|ccc|ccc@{}}
\toprule
 & \multicolumn{3}{c|}{TriviaQA} & \multicolumn{3}{c|}{NQ-Open} & \multicolumn{3}{c|}{PopQA} & \multicolumn{3}{c}{HotpotQA} \\
$\alpha$ & Point. & Cons. & $\Delta$ & Point. & Cons. & $\Delta$ & Point. & Cons. & $\Delta$ & Point. & Cons. & $\Delta$ \\
\midrule
.05 & .12 & .00 & $-$.12 & .00 & .00 & .00 & .00 & .00 & .00 & .00 & .00 & .00 \\
.10 & .63 & .01 & $-$.62 & .02 & .00 & $-$.02 & .87 & .00 & $-$.87 & .04 & .00 & $-$.04 \\
.15 & .91 & .64 & $-$.27 & .24 & .00 & $-$.24 & 1.0 & .93 & $-$.07 & .49 & .00 & $-$.49 \\
.20 & 1.0 & .93 & $-$.07 & .82 & .05 & $-$.76 & 1.0 & 1.0 & .00 & .74 & .30 & $-$.45 \\
.25 & 1.0 & 1.0 & .00 & 1.0 & .80 & $-$.20 & 1.0 & 1.0 & .00 & .96 & .72 & $-$.24 \\
\bottomrule
\end{tabular*}
\end{table}

\section{Analysis}
\label{app:analysis}
This section moves beyond aggregate metrics to analyze how retrieval improves evaluation, including its impact on judge uncertainty, examples where it corrects incorrect verdicts, and the residual failure modes.

\subsection{Effect of Retrieval on Uncertainty}
\label{app:deltau}
Retrieval does not lower the judge's uncertainty uniformly. Table~\ref{tab:deltau} reports, per dataset, the fraction of items on which Mode~2 uncertainty is higher than Mode~1 ($U_2 > U_1$) versus lower, with the mean signed change $\Delta U = U_2 - U_1$ within each group. On most items retrieval reduces uncertainty, but on 14--29\% it raises it, concentrated on the multi-hop HotpotQA and NQ-Open, where retrieved snippets are only partially relevant or sources disagree. These are exactly the items the second threshold~$\hat{t}_2$ routes to abstention when retrieval fails to resolve the judge's uncertainty.

\begin{table}[!ht]
\centering
\small
\caption{Effect of retrieval on judge uncertainty. Results are obtained with Qwen3-8B as both candidate and judge, the LLM evaluator, and $k\!=\!3$ retrieval using predictive entropy (PE). The ``raised'' and ``lowered'' columns report the fraction of instances where retrieval increases or decreases uncertainty, respectively. Mean~$\Delta U$ is averaged within each group.}
\label{tab:deltau}
\begin{tabular*}{\textwidth}{@{\extracolsep{\fill}}l|cccc@{}}
\toprule
 & TriviaQA & NQ-Open & PopQA & HotpotQA \\
\midrule
\% raised ($U_2\!>\!U_1$)  & 20.1 & 27.9 & 14.3 & 29.1 \\
\quad mean $\Delta U$      & $+0.10$ & $+0.11$ & $+0.10$ & $+0.12$ \\
\% lowered ($U_2\!<\!U_1$) & 68.0 & 68.8 & 84.6 & 69.7 \\
\quad mean $\Delta U$      & $-0.07$ & $-0.09$ & $-0.10$ & $-0.12$ \\
\bottomrule
\end{tabular*}
\end{table}

\subsection{Qualitative Examples}
\label{app:qualitative}

Tables~\ref{tab:qualitative} and~\ref{tab:evidence_examples} show representative cases where the judge issues an incorrect verdict in Mode~1 but corrects it after retrieval in Mode~2. In each case, Mode~1 confidently hallucinated while retrieved evidence grounded the judge in verifiable facts.

\begin{table*}[!ht]
\centering
\caption{Examples where retrieval corrects the judge's verdict. Candidate: Qwen3-8B, Judge: Qwen3-14B.}
\label{tab:qualitative}
\resizebox{\textwidth}{!}{%
\small
\begin{tabular}{p{3.5cm}p{2cm}p{1.5cm}p{5.5cm}p{5.5cm}}
\toprule
\textbf{Question} & \textbf{Candidate} & \textbf{Gold} & \textbf{Mode 1 (Wrong)} & \textbf{Mode 2 (Corrected)} \\
\midrule
Who voices Jack Kahuna Laguna in SpongeBob SquarePants? & Steve Carell & Johnny Depp & \textit{``Steve Carell provided the voice\ldots This is a well-documented fact.''} \newline \textbf{Verdict: True} & \textit{``The web search results indicate that Johnny Depp, not Steve Carell, provided the voice\ldots''} \newline \textbf{Verdict: False} \\
\midrule
Which North African country's flag is just the colour green? & Chad & Libya & \textit{``The flag of Chad is indeed a single solid green rectangle\ldots''} \newline \textbf{Verdict: True} & \textit{``The search results indicate that Libya's flag was a single green field\ldots''} \newline \textbf{Verdict: False} \\
\midrule
`The Great American Chocolate Bar' is better known as? & Milky Way & Hershey & \textit{``The Great American Chocolate Bar is indeed known as the Milky Way bar\ldots''} \newline \textbf{Verdict: True} & \textit{``The search results indicate that it is a nickname for the Hershey Milk Chocolate Bar\ldots''} \newline \textbf{Verdict: False} \\
\bottomrule
\end{tabular}%
}
\end{table*}

\begin{table*}[!ht]
\centering
\caption{Retrieved evidence ($k\!=\!3$) for the examples in Table~\ref{tab:qualitative}. The first result in each case directly contains the correct answer.}
\label{tab:evidence_examples}
\resizebox{\textwidth}{!}{%
\small
\begin{tabular}{p{3.5cm}cp{12cm}}
\toprule
\textbf{Question} & \textbf{\#} & \textbf{Retrieved Snippet} \\
\midrule
\multirow{3}{3.5cm}{Who voices Jack Kahuna Laguna in SpongeBob?}
 & 1 & \textit{``\ldots American actor and musician Johnny Depp guest starred in the episode as the voice of Jack Kahuna Laguna, a surf guru\ldots''} --- Wikipedia \\
 & 2 & \textit{``Jack Kahuna Laguna has a striking resemblance to his voice actor, Johnny Depp.''} --- Fandom \\
 & 3 & \textit{``\ldots it guest starred Johnny Depp as the legendary surfer, Jack Kahuna Laguna.''} --- IMDb \\
\midrule
\multirow{3}{3.5cm}{Which N.\ African country's flag is just green?}
 & 1 & \textit{``The flag of the Libyan Arab Jamahiriya\ldots consisted of a simple green field. It was the only national flag\ldots with only one colour.''} --- Wikipedia \\
 & 2 & \textit{``Republic of Ireland and the Ivory Coast the flags are the same\ldots''} --- Facebook \\
 & 3 & \textit{``Green is an important colour in Islam\ldots''} --- Reddit \\
\midrule
\multirow{3}{3.5cm}{`The Great American Chocolate Bar' is better known as?}
 & 1 & \textit{``The classic HERSHEY'S Milk Chocolate bar has been referred to as `The Great American Chocolate Bar.'\,''} --- Hershey \\
 & 2 & \textit{``Hershey refers to it as `The Great American Chocolate Bar'. First sold in 1900.''} --- Wikipedia \\
 & 3 & \textit{``The Great American Chocolate Bar campaign served the company well\ldots''} --- Hershey Archives \\
\bottomrule
\end{tabular}%
}
\end{table*}

\subsection{Failure Modes}
\label{app:failuremodes}
When retrieval does not fix an incorrect Mode~1 verdict, the residual errors fall into four qualitative classes. (i)~\emph{Off-topic retrieval}: an ambiguous or under-specified question yields search queries that miss the relevant evidence. (ii)~\emph{Insufficient evidence}: the retrieved snippets are only partially relevant and lack the specific fact needed to verify the candidate. (iii)~\emph{Judge reasoning error}: the snippets contain the relevant evidence, but the judge still returns a wrong verdict. For example defaulting to \texttt{False} in the absence of an exact lexical match. (iv)~\emph{Admission function disagreement}: the candidate is semantically close to the gold answer but the admission function scores it as different. The two-threshold policy sends many such cases to abstention rather than accepting them, which is why coverage rather than FDR absorbs the difficulty on the hardest datasets.


\end{document}